\RequirePackage{amsmath}
\RequirePackage{amssymb}
\RequirePackage{amsfonts}

\documentclass[]{fairmeta}

\usepackage{amsthm}
\newtheorem{proposition}{Proposition}

\setcitestyle{numbers,square,comma}

\graphicspath{{figures/}}

\title{Distillation of Synthetic Data for Time Series Foundation Models}

\author[1]{Niloy Biswas}
\author[1]{Noureddine El Karoui}

\affiliation[1]{Meta AI}

\abstract{Time series foundation models (TSFMs) are increasingly pre-trained on synthetically generated time series
trajectories, where the data generating process is known. Current pre-training recipes are based on loss objectives
which compare TSFM outputs to realized future values of each trajectory.
We instead propose loss objectives which compare TSFM outputs to the conditional forecast distribution of each trajectory,
a procedure we call synthetic data distillation (SDD).
SDD corresponds to a Rao-Blackwellization of the training objective, in that it leaves the expectation of stochastic gradients
unchanged while provably reducing the covariance of the stochastic gradient under the Loewner partial ordering.
We empirically validate SDD on a TSFM model family of sizes from $4$M to $2.5$B parameters, and observe faster convergence of validation loss at every model size:
on Gaussian Process data, SDD attains or improves upon the Status Quo loss whilst requiring $10\%-40\%$ less training iterations.}
\correspondence{\email{niloy@meta.com}} 
\code{\url{https://github.com/niloyb/sdd}}

\begin{document}

\maketitle
\section{Synthetic data for Time Series Foundation Models pre-training}

Time series foundation models (TSFMs) are large neural networks pre-trained on a diverse corpus of time series data across different domains to perform time series forecasting.
TSFMs process numeric historical time series values as context \citep{nie2023patchtst,vaswani2017attention}, similar to how text tokens are processed as context by Large Language Models. They then predict future time series values on time series data unseen in the training process. TSFMs 
enable zero-shot forecasting and are an active research area \citep{das2024timesfm, chronos2, tirex2, toto2}.

Compared to text, publicly available real-world time series are scarce.
This has catalyzed efforts to develop synthetic data generation techniques for TSFM pre-training \citep{dooley2023forecastpfn,ansari2024chronos,das2024timesfm,xie2025cauker,moroshan2025tempopfn}, 
and synthetic datasets are now a core component of state-of-the-art TSFMs \citep{dooley2023forecastpfn, moroshan2025tempopfn, chronos2, das2024timesfm, moirai2, tirex2, toto2}.

\paragraph{Our Contributions.} We introduce synthetic data distillation (SDD), which is motivated by the literature on \textit{dataset distillation} \cite{wang2018dataset}. 
SDD condenses data from infinitely many synthetic time series trajectories into a single loss calculation without generating the trajectories. 
SDD accelerates the pre-training of TSFMs by reducing the variance of the stochastic gradients during training. 
Section \ref{sec:sdd} introduces SDD and Section \ref{sec:experiments} highlights its empirical benefits.
Our work participates in the wider effort to pre-train structured data foundation models using synthetic data
\citep{muller2022pfn, dooley2023forecastpfn, toto2, moroshan2025tempopfn, Hollmann2025}.

\section{Distillation of Synthetic Data for Time Series Foundation Models} \label{sec:sdd}
\subsection{Status Quo uses a single realized trajectory as the target during pre-training} \label{sec:status}
Let $f_{\theta}$ denote a TSFM with weights $\theta$. Current pre-training is based on loss functions of the form
\begin{equation}
\label{eq:cpm}
 \frac{1}{B} \sum_{b=1}^{B} \ell\big(f_\theta(y^{(b)}_{0:t}),\, y^{(b)}_{t+1:t+h}\big)
\end{equation}
where $\ell$ is the loss function, $B$ is the batch size, $h$ is the forecast horizon,
$y^{(b)}_{0:t}$ and $y^{(b)}_{t+1:t+h}$ are the historical and future time series values of the $b^{th}$ sample respectively.
Common loss functions $\ell$ include mean squared error (MSE), pinball (a.k.a.\ quantile) loss and cross-entropy loss when TSFMs output point forecasts, quantile forecasts, and a categorical distribution of bins respectively.

\subsection{Synthetic Data Distillation uses the exact conditional distribution of future trajectories}
This section develops Synthetic Data Distillation (SDD). Consider a single time series trajectory $(y_{t})_{t \geq 0}$ as in \eqref{eq:cpm}, where we drop the batch index notation for clarity. This trajectory is generated from some data-generating process $\pi_\alpha$ with hyperparameters $\alpha$. For real-world data, $\pi_\alpha$ is unknown; for synthetic data, $\pi_\alpha$ is from some pre-specified prior distribution \citep{muller2022pfn} and known at generation time. For example, $\pi_\alpha$ might correspond to a Gaussian process and $\alpha$ its mean and kernel functions.

In status quo pre-training, the loss function \eqref{eq:cpm} uses the realized future values $y_{t+1:t+h}$ as the ground truth during training, where $y_{t+1:t+h}$ is a single realized trajectory from the conditional distribution of $\pi_\alpha$ given observed history $y_{0:t}$. For a large class of synthetically generated time series, this conditional distribution is known and analytically tractable.
SDD proposes to make full use of this conditional distribution, instead of just utilizing a single realized trajectory $y_{t+1:t+h}$ of future values.

\paragraph{A stylized example.} Figure \ref{fig:status_quo_sdd} compares Status Quo and SDD on a stylized example of a point forecast TSFM. Status Quo loss is equal to MSE of the TSFM forecast $f_\theta(y_{0:t})$ with a single realized future trajectory $y_{t+1:t+h}$ as the ground truth. SDD loss is equal (up to a constant that does not depend on $\theta$, see Table \ref{tab:distilled-general}) to MSE with the conditional mean of future trajectories $\mathbb{E}_{\pi_\alpha}[Y_{t+1:t+h} | y_{0:t}]$ as the ground truth. This conditional mean is known exactly for a large class of synthetic data generators. SDD distills the entire conditional distribution of future trajectory into a single loss calculation, and reduces the variance of the loss function compared to the Status Quo ground truth which is based on just a single realization of future values.
It is numerically equivalent to generating infinitely many $y_{t+1:t+h}$ given $y_{0:t}$ and calculating mean of $\ell\big(f_\theta(y_{0:t}),\, y_{t+1:t+h}\big)$.
As the integration can be done in closed form, none of the samples actually need to be generated.
\begin{figure}[ht]
  \centering
  \includegraphics[width=\linewidth]{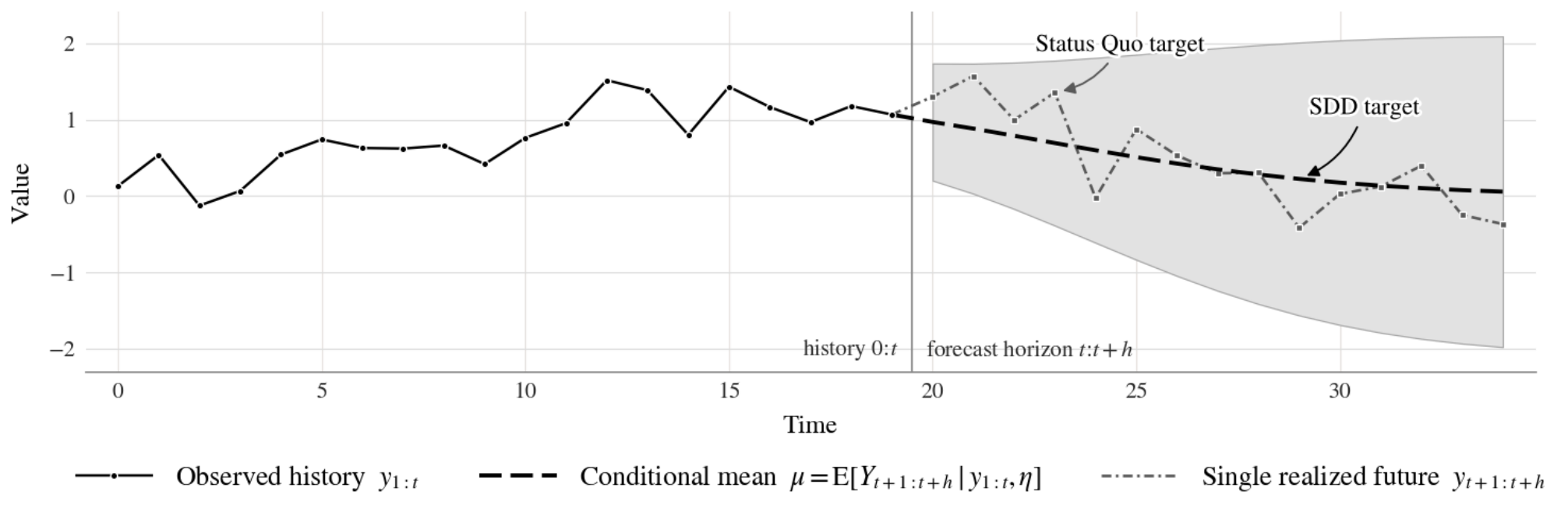}
  \caption{Comparison of Status Quo and SDD on a single trajectory for training a point forecast TSFM with MSE loss. Status Quo calculates MSE with realized future trajectory $y_{t+1:t+h}$ as the ground truth. SDD calculates (up to an additive constant independent of $\theta$) MSE with the conditional mean of future trajectory $\mathbb{E}_{\pi_\alpha}[Y_{t+1:t+h} | y_{0:t}]$ as the ground truth.
  }
  \label{fig:status_quo_sdd}
\end{figure}
\subsection{General methodology and implementation considerations}
\label{sec:method}
We now develop SDD in greater generality for any loss function. Consider forecast error loss of a single time series trajectory $(y_{t})_{t \geq 0}$ as in \eqref{eq:cpm}. Note that
\begin{equation}
\label{eq:cpm-distill}
\mathbb{E}_{\pi_\alpha} \big[ \ell\big(f_\theta(y_{0:t}),\, y_{t+1:t+h}\big) \big] =
\mathbb{E}_{\pi_\alpha} \big[ \mathbb{E}_{\pi_\alpha} \big[\ell\big(f_\theta(y_{0:t}),\, y_{t+1:t+h}\big) \big| y_{0:t}\big] \big].
\end{equation}
For many common synthetic data generating processes and many loss functions, the \emph{distilled loss}
\begin{equation}
\label{eq:sdd-loss}
\ell_{distill}^{(\pi_\alpha, h)} \big( \theta, y_{0:t} \big)
 := \mathbb{E}_{\pi_\alpha} \big[ \ell\big(f_\theta(y_{0:t}),\, y_{t+1:t+h}\big) \,\big|\, y_{0:t} \big]
\end{equation}
is analytically tractable. Synthetic Data Distillation proposes to replace $\ell$ with $\ell_{distill}$ in \eqref{eq:cpm}.
By \eqref{eq:cpm-distill}, $\ell_{distill}$ has the same expectation as $\ell$ with the sampling
noise of the realized future integrated out.
Table \ref{tab:distilled-general} contains expressions for the distilled loss $\ell_{distill}$ for some common loss functions.
\begin{table}[t] \centering
\small
\begin{tabular}{@{}lll@{}} \toprule Loss $\ell$ & Realized form & Distilled form $\ell_{distill}$ \\ \midrule Squared error & $(\hat y_{t+i} - Y_{t+i})^2$ & $(\hat y_{t+i} -\mu)^2 + s^2$ \\[2pt]
Absolute error & $|\hat y_{t+i} - Y_{t+i}|$ & $s\!\left[\,2\,m(z) + z\big(2F_0(z)-1\big)\right]$ \\[2pt]
Pinball at level $\tau$ & $(Y_{t+i}-\hat y_{t+i})\big(\tau-\mathbf 1\{Y_{t+i}<\hat y_{t+i}\}\big)$ & $s\!\left[\,m(z) + z\big(F_0(z)-\tau\big)\right]$ \\[2pt]
Cross-entropy over bins & $-\log \hat p(\mathrm{bin}(Y_{t+i}))$ & $-\sum_k \mathbb{P}(Y_{t+i}\in\mathrm{bin}_k)\,\log \hat p_k$ \\ \bottomrule
\end{tabular}
\caption{
$\hat y_{t+i}$ is the TSFM prediction; $Y_{t+i}$ is the random variable of the trajectory at time $t+i$; $\mu=\mathbb{E}_{\pi_\alpha}[Y_{t+i} | y_{0:t}]$ is the conditional mean of $Y_{t+i}$ given $y_{0:t}$; $s>0$ is the conditional standard deviation; and $z=(\hat y_{t+i}-\mu)/s$. The standardized variable $Z=(Y_{t+i}-\mu)/s$ has density $f_0$ and CDF $F_0$, and upper partial mean $m(z)=\int_z^{\infty}u\,f_0(u)\,du$. $\hat p_k$ is predicted probability of bin partitions. 
}
\label{tab:distilled-general}
\end{table}

\paragraph{Tractability.} SDD requires the knowledge of the quantities such as mean, variance and CDF of conditional distribution of the multi-step future $Y_{t+1:t+h}$ given the history $y_{0:t}$.
Note that the losses of Table~\ref{tab:distilled-general} are
sums of per-time-point terms, so only the \emph{marginal} conditional law of each $Y_{t+i}$ is
needed, not the joint law across the horizon. A large class of synthetic time series generators are tractable in this sense, including: (i) Gaussian Processes (GPs); (ii) Linear Gaussian State Space Models (e.g.\ ARIMA, Dynamic Linear Models); (iii) Linear and Log-Linear Stochastic Differential Equations (e.g.\ Ornstein--Uhlenbeck processes and Geometric Brownian Motion), and  (iv) all independent non-identically distributed (i.n.i.d.) time series $(y_t)_{t \geq 0}$ of the form
$y_t := g(t, \epsilon_t)$, where $g$ is a known deterministic function and the $\epsilon_t$ are
independent noise.
Appendix~\ref{app:closedforms} gives the distilled loss 
for each of these families and each loss of Table~\ref{tab:distilled-general}.
Two examples of intractable synthetic data generators are structural causal models with non-linear activations \cite{xie2025cauker} and
nonlinear State Space Models. 
Note that GPs and i.n.i.d.\ time series (e.g.\ with change-points, sawtooths or spikes-based deterministic functions and additive noise) already capture most of the current synthetic data generators used in TSFM pre-training \citep{moroshan2025tempopfn, toto2}. Appendix
\ref{app:generators} reviews and classifies all existing synthetic generators. 

For intractable synthetic data generators, it is possible to calculate Monte-Carlo estimates of $\ell_{distill}$ by sampling many trajectories of $Y_{t+1:t+h}$ given history $y_{0:t}$ for each $t$. This provides a balance between Status Quo (just a single trajectory and higher variance) and SDD (numerically equivalent to infinite trajectories and lower variance, but requires tractability).
When some of the pre-training data is intractable (e.g. real-world data), we can also fallback to the status quo loss just for the intractable time series. 
This means SDD remains useful whenever some portion of the pre-training corpus are from tractable synthetic data generators.

\paragraph{Implementation considerations for efficient pre-training.} Teacher-forcing \cite{das2024timesfm} or contiguous patch masking (CPM) \cite{moirai2, chronos2, toto2} are commonly used for TSFM pre-training. For time series of length $T$, the SDD teacher-forcing loss is given by $ \frac{1}{T-h-s+1}\sum_{t=s}^{T-h} \ell_{distill}^{(\pi_\alpha, h)} \big( \theta, y_{0:t} \big)$ for some minimum history threshold $s \geq 0$; the SDD CPM loss is given by
$\frac{1}{M} \sum_{i=1}^M \ell_{distill}^{(\pi_\alpha, h_i)} \big( \theta, y_{0:s_i} \big)$, where $(s_i+1,s_i+h_i)_{i=1}^M$ are the randomly sampled masked periods. To implement SDD efficiently, we can calculate the conditional distribution quantities in an auto-regressive manner when generating the synthetic data. For example, when generating and saving $y_{1:t}$, we can also save the quantities such as mean, variance of $Y_{t+1:T}$ given $y_{0:t}$. 
Then after the pre-training data is cached, SDD and Status Quo cost the same per training iteration.

\subsection{Variance Reduction Guarantees}
\label{sec:theory}

Denote $L(\theta) = \ell\big(f_\theta(y_{0:t}), y_{t+1:t+h}\big)$ and
$L_{distill}(\theta) = \ell_{distill}^{(\pi_\alpha, h)} \big( \theta, y_{0:t} \big)$ for the status quo and distilled loss respectively. Proposition~\ref{prop:rb} establishes a Rao--Blackwellization result \citep{rao1945, blackwell1947} for SDD. 

\begin{proposition}[Rao--Blackwellization]
\label{prop:rb}
Suppose $L(\theta)$ is differentiable in $\theta$ with $\mathbb{E}_{\pi_\alpha}\lVert\nabla_\theta L(\theta)\rVert^2 < \infty$,
and that differentiation and conditional expectation may be interchanged, so that
$\nabla_\theta L_{distill}(\theta) = \mathbb{E}_{\pi_\alpha}[\nabla_\theta L(\theta) \mid y_{0:t}]$. Then $\mathbb{E}_{\pi_\alpha}[\nabla_\theta L_{distill}(\theta)]=\mathbb{E}_{\pi_\alpha}[\nabla_\theta L(\theta)]$, and, in the Loewner partial ordering on positive semi-definite matrices,
\begin{equation}
\label{eq:cov}
\operatorname{Cov}_{\pi_\alpha}[\nabla_\theta L_{distill}(\theta)]
=\operatorname{Cov}_{\pi_\alpha}[\nabla_\theta L(\theta)]
 - \mathbb{E}_{\pi_\alpha}[\operatorname{Cov}_{\pi_\alpha}(\nabla_\theta L(\theta) \,|\, y_{0:t})]
 \preceq \operatorname{Cov}_{\pi_\alpha} [\nabla_\theta L(\theta)].
\end{equation}
\end{proposition}
\begin{proof}
By tower property $\mathbb{E}_{\pi_\alpha} [ \nabla_\theta L_{distill}(\theta) ]=\mathbb{E}_{\pi_\alpha}[\mathbb{E}_{\pi_\alpha}[\nabla_\theta L(\theta) \mid y_{0:t}] ]=\mathbb{E}_{\pi_\alpha}[\nabla_\theta L(\theta)]$, and \eqref{eq:cov} is the law of total covariance applied to $\nabla_\theta L(\theta)$ given $y_{0:t}$, since
$\mathbb{E}_{\pi_\alpha}[\nabla_\theta L(\theta) \mid y_{0:t}] = \nabla_\theta L_{distill}(\theta)$. The subtracted term is an expectation of covariance matrices and hence positive semi-definite \cite{bhatia1997matrix}.
\end{proof}
Proposition~\ref{prop:rb} shows that SDD produces unbiased and lower variance gradients for back-propagation during TSFM pre-training. 
This in turn yields faster convergence, as next highlighted in Section \ref{sec:experiments}.

\section{Numerical Experiments}
\label{sec:experiments}

We pre-train the five Toto-2 architectures \citep{toto2}, spanning $4$M to $2.5$B parameters, from random initialization on trajectories from univariate Gaussian processes of length $T=512$. In this setting, the conditional forecast distributions are Gaussian and the distilled losses of Table~\ref{tab:distilled-general} are analytically tractable. Every model carries a quantile head over the nine deciles and is pre-trained with contiguous patch masking under pinball loss: the Status Quo arm scores the realized future values while the SDD arm scores the distilled pinball loss of Table~\ref{tab:distilled-general}. We train for $300{,}000$ steps at batch $B=16$ in a single-pass regime, such that no trajectory is revisited during training, and calculate the next patch
(corresponding to the next $P=32$ time-points) Continuous Ranked Probability Score (CRPS) on held-out time series to track the validation loss of Status Quo and SDD. 

Figure \ref{fig:flops-crps} plots that validation loss against cumulative training compute, estimated as $6ND$ for $N$ parameters and $D$ patch tokens, so that a model's own trajectory and the envelope across model sizes can be read on one axis. 
Here a patch is the TSFM's token: each of the $T/P = 16$ patches of a series is one
sequence position, so $D$ advances by $B \cdot T/P = 256$ per optimizer step. 
The conditional distribution moments needed for SDD are cached beforehand when the synthetic data is generated, so during training SDD and Status Quo steps cost the same
(Section~\ref{sec:method}). 
Figure \ref{fig:flops-crps} shows that SDD attains the same validation loss as Status Quo while spending $38$-$46\%$ fewer FLOPs, a convergence speed-up of $1.6\times$ to $1.85\times$ depending on model size. It also shows that SDD consistently attains lower validation loss than Status Quo given the same training compute.
Appendix~\ref{app:details} details the data generation and training run configurations, and Appendix~\ref{app:extra} reports additional experiments under teacher forcing and with varying observation noise, batch size and sequence length.

\begin{figure}[ht]
  \centering
  \includegraphics[width=\linewidth]{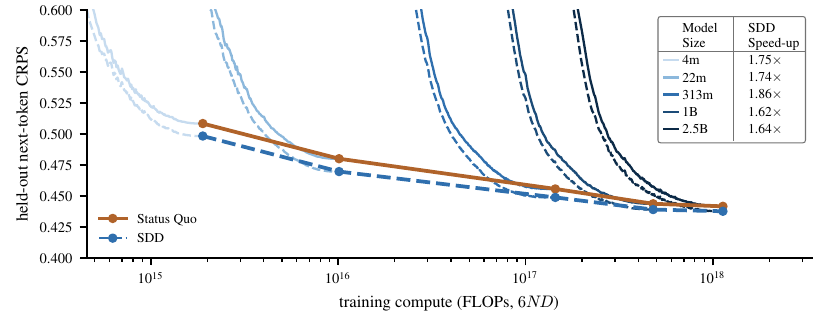}
  \caption{Held-out next-token CRPS of Status Quo and SDD against training compute FLOPs, 
  for the five Toto-2 architectures \citep{toto2} trained from random initialization. 
  Each colored line is one model's compute curve, solid for Status Quo and
  dashed for SDD. The two heavy marked lines join the end points across sizes. 
  \emph{SDD speed-up} is $300{,}000$ divided by the number of iterations required by SDD to attain the $300{,}000^{\text{th}}$ step Status Quo loss. E.g. for the $1$B model SDD attained the Status Quo loss in $185{,}000$ training steps, giving a $\frac{300{,}000}{185{,}000} \approx 1.62\times$ speed up or equivalently a compute saving of $38\%$. 
  }
  \label{fig:flops-crps}
\end{figure}

\section{Discussion and Future Work}
We have introduced Synthetic Data Distillation as a way to accelerate the pre-training of time series foundation models with synthetic data. 
SDD motivates future work on: (i) numerical experiments for a wider range of synthetic data generators, (ii) whether that compute saving persists at the scale and data mixtures of production TSFM pre-training runs, where synthetic data is only part of the corpus, (iii) applications of SDD beyond TSFM to Tabular Foundation models \cite{hollmann2023tabpfn, Hollmann2025}.

\bibliographystyle{assets/plainnat}
\bibliography{references}

\beginappendix

\section{Distilled Losses}
\label{app:losses}

This appendix derives each row of Table~\ref{tab:distilled-general}.

Fix a history $y_{0:t}$ and a horizon index $i \in \{1, \dots, h\}$. All expectations are under the
generator $\pi_\alpha$ and conditional on $y_{0:t}$; we assume $\mathbb{E}[Y_{t+i}^2] < \infty$, so
that every quantity below exists. Write $\mu$ and $s > 0$ for the conditional mean and standard
deviation of $Y_{t+i}$, and set
\begin{equation*}
z = \frac{\hat y_{t+i} - \mu}{s}, \qquad Z = \frac{Y_{t+i} - \mu}{s},
\end{equation*}
so that $Z$ has mean $0$ and variance $1$. Let $f_0$ and $F_0$ be its density and CDF and let
$m(z) = \int_z^\infty u\, f_0(u)\,du$ be its upper partial mean; since $\mathbb{E}[Z] = 0$ we also
have $\int_{-\infty}^z u\, f_0(u)\,du = -m(z)$. The generator is Gaussian in all of our experiments,
in which case $f_0 = \phi$ and $F_0 = \Phi$ are the standard normal density and CDF and
$m(z) = \phi(z)$.

Every loss in Table~\ref{tab:distilled-general} is a sum over the horizon indices $i = 1, \dots, h$
of the scalar term shown, up to a factor $1/h$ when $\ell$ is a mean rather than a sum. Conditional
expectation is linear, so the distilled loss \eqref{eq:sdd-loss} aggregates the $\ell_{distill}$ below in
exactly the same way, and only the \emph{marginal} conditional law of each $Y_{t+i}$ is ever
required.

\paragraph{Squared error.} Expanding around $\mu$,
$\mathbb{E}[(\hat y_{t+i} - Y_{t+i})^2] = (\hat y_{t+i}-\mu)^2 - 2(\hat y_{t+i}-\mu)\,\mathbb{E}[Y_{t+i}-\mu] + \mathbb{E}[(Y_{t+i}-\mu)^2]$.
The middle term is zero and the last is $s^2$, which does not depend on $\theta$. The distilled loss
is squared error against the conditional mean, up to an additive constant.

\paragraph{Absolute error.} $|\hat y_{t+i} - Y_{t+i}| = s\,|z - Z|$. Splitting the integral at $z$,
\begin{equation*}
\mathbb{E}\,|z - Z| = \int_{-\infty}^{z} (z-u) f_0(u)\,du + \int_{z}^{\infty} (u-z) f_0(u)\,du
 = 2 m(z) + z\big(2F_0(z) - 1\big),
\end{equation*}
so $\ell_{distill} = s\left[\,2 m(z) + z\big(2F_0(z)-1\big)\right]$.

\paragraph{Pinball at level $\tau$.} The pinball loss is
$(Y_{t+i}-\hat y_{t+i})(\tau - \mathbf 1\{Y_{t+i} < \hat y_{t+i}\}) = s\,(Z-z)(\tau - \mathbf 1\{Z<z\})$.
Using $\mathbb{E}[Z-z] = -z$ and $\mathbb{E}[(Z-z)\mathbf 1\{Z<z\}] = -m(z) - z F_0(z)$, the
expectation is $\ell_{distill} = s[\,m(z) + z(F_0(z)-\tau)]$. Its derivative in $\hat y_{t+i}$ is
$F_0(z) - \tau$, since $m'(z) = -z f_0(z)$ cancels the $z f_0(z)$ term and
$dz/d\hat y_{t+i} = 1/s$. The second derivative is $f_0(z)/s \ge 0$, so $\ell_{distill}$ is convex and,
whenever $F_0$ is strictly increasing, is minimized at $\hat y_{t+i} = \mu + s F_0^{-1}(\tau)$, the
conditional $\tau$-quantile of $Y_{t+i}$ given $y_{0:t}$. Training against $\ell_{distill}$ therefore
targets the same quantiles as training against the realized pinball loss.

\paragraph{Cross-entropy over bins.} Since
$-\log \hat p(\mathrm{bin}(Y_{t+i})) = -\sum_k \mathbf 1\{Y_{t+i} \in \mathrm{bin}_k\} \log \hat p_k$,
linearity of expectation gives $\ell_{distill} = -\sum_k \mathbb{P}(Y_{t+i} \in \mathrm{bin}_k) \log \hat p_k$.
The distilled loss is the cross-entropy of the predicted bin distribution against the exact
conditional bin distribution, in place of the one-hot bin of a single realized draw.

\section{Tractable Distilling Losses}
\label{app:tractable}

Section~\ref{sec:method} names four families of synthetic generator whose conditional forecast law is
closed-form. This appendix writes out the distilled loss \eqref{eq:sdd-loss} for every one of
those families and every loss of Table~\ref{tab:distilled-general}.

\subsection{Closed Forms by Generator Family}
\label{app:closedforms}

The grid factorizes, which is what makes it small enough to tabulate. Fix a horizon index
$i \in \{1, \dots, h\}$ and write, all conditional on $y_{0:t}$ and under $\pi_\alpha$,
\begin{equation}
\label{eq:musfu}
\mu_i = \mathbb{E}\big[Y_{t+i}\big], \quad
s_i^2 = \operatorname{Var}\big(Y_{t+i}\big), \quad
F_i(x) = \mathbb{P}\big(Y_{t+i} \le x\big), \quad
U_i(x) = \mathbb{E}\big[(Y_{t+i}-x)^{+}\big],
\end{equation}
where $(a)^+ = \max(a,0)$ and $U_i$ is the \emph{upper partial expectation}. The generator enters the
distilled loss only through the four functionals \eqref{eq:musfu}, and the loss only through how it
combines them.

\begin{proposition}[Distilled losses in terms of $(\mu, s, F, U)$]
\label{prop:musfu}
Fix $i$, drop it from the notation, and write $\hat y = \hat y_{t+i}$. Suppose
$\mathbb{E}[Y^2 \mid y_{0:t}] < \infty$. Then the distilled terms of Table~\ref{tab:distilled-general}
are
\begin{align}
\text{squared error:}\quad & \ell_{distill}(\hat y) = (\hat y - \mu)^2 + s^2, \label{eq:d-se}\\
\text{absolute error:}\quad & \ell_{distill}(\hat y) = (\hat y - \mu) + 2\,U(\hat y), \label{eq:d-ae}\\
\text{pinball at level } \tau:\quad & \ell_{distill}(\hat y) = U(\hat y) + (1-\tau)(\hat y - \mu), \label{eq:d-pin}\\
\text{cross-entropy over bins:}\quad & \ell_{distill}(\hat p) = -\sum_{k=1}^{K} \big(F(e_k) - F(e_{k-1})\big) \log \hat p_k, \label{eq:d-ce}
\end{align}
where $-\infty = e_0 < e_1 < \dots < e_K = +\infty$ are the bin edges, so that
$\mathrm{bin}_k = (e_{k-1}, e_k]$.
\end{proposition}

\begin{proof}
Equation \eqref{eq:d-se} is Appendix~\ref{app:losses}. For \eqref{eq:d-ae}, put $a = \hat y - Y$ and use
$|a| = a + 2(-a)^{+}$, so $\mathbb{E}|\hat y - Y| = (\hat y - \mu) + 2\,\mathbb{E}[(Y - \hat y)^{+}]$.
For \eqref{eq:d-pin}, the pinball loss is
$\tau (Y-\hat y)^{+} + (1-\tau)(\hat y - Y)^{+}$, and $(\hat y - Y)^{+} = (Y - \hat y)^{+} - (Y - \hat y)$,
so its expectation is $U(\hat y) - (1-\tau)(\mu - \hat y)$. Equation \eqref{eq:d-ce} is linearity of
expectation applied to the indicator of each bin.
\end{proof}

Proposition~\ref{prop:musfu} is equivalent to Table~\ref{tab:distilled-general}: writing
$U(\hat y) = s\big[m(z) - z\big(1 - F_0(z)\big)\big]$ with $z = (\hat y - \mu)/s$ recovers the
standardized forms there. Two consequences are worth noting. Only $\mu$ enters
\eqref{eq:d-se} up to the $\theta$-free constant $s^2$, so the distilled squared error needs only the
conditional mean to train. And $\ell_{distill}$ in \eqref{eq:d-pin} has derivative $F(\hat y) - \tau$ in
$\hat y$, so it is minimized at the exact conditional $\tau$-quantile whatever the conditional law is:
distillation does not change what the quantile head is asked to learn.

Table~\ref{tab:tractable} completes the picture. Panel (a) gives $(\mu_i, s_i^2, F_i, U_i)$ for each
of the four tractable families; panel (b) evaluates \eqref{eq:d-se}--\eqref{eq:d-ce} for each
conditional law that arises. To read off any one of the sixteen distilled losses, take the family's row
in (a) and the loss's row in (b) under the matching law.

\begin{table}[p]
  \centering
  \footnotesize
  \setlength{\tabcolsep}{4pt}
  \caption{Exact distilled losses for the tractable generator families of
  Section~\ref{sec:method}. \emph{(a)} the conditional forecast law of $Y_{t+i}$ given $y_{0:t}$,
  summarized by the four functionals \eqref{eq:musfu}; $\delta_i$ is the elapsed time from $t$ to
  $t+i$. \emph{(b)} the distilled loss $\ell_{distill}$ of Proposition~\ref{prop:musfu} for each law that
  arises in (a), with $\phi$ and $\Phi$ the standard normal density and CDF. Every family in (a) is
  Gaussian except the log-linear SDE, which is lognormal, and the general i.n.i.d.\ case, which is
  whatever the noise law makes it.}
  \label{tab:tractable}

  \textbf{(a) Conditional forecast law $Y_{t+i} \mid y_{0:t}$}\\[2pt]
  \begin{tabular}{@{}p{0.155\linewidth}p{0.275\linewidth}p{0.10\linewidth}p{0.415\linewidth}@{}}
    \toprule
    Family & Generator & Law & Conditional moments and distribution\\
    \midrule
    Linear Gaussian state space (ARIMA, DLM)
      & $x_{u+1} = A_u x_u + b_u + w_u$, $w_u \sim \mathcal{N}(0, Q_u)$;\newline
        $y_u = c_u^\top x_u + d_u + v_u$, $v_u \sim \mathcal{N}(0, R_u)$
      & Gaussian
      & $\mu_i = c_{t+i}^\top m_i + d_{t+i}$, \quad $s_i^2 = c_{t+i}^\top P_i\, c_{t+i} + R_{t+i}$,\newline
        with $m_i = A_{t+i-1} m_{i-1} + b_{t+i-1}$ and\newline
        $P_i = A_{t+i-1} P_{i-1} A_{t+i-1}^\top + Q_{t+i-1}$ run forward from the Kalman filter
        $(m_0, P_0) = (\hat x_{t\mid t}, \hat P_{t \mid t})$\\
    \addlinespace
    Gaussian process (KernelSynth, GP)
      & $y_u = g(u) + f(u) + v_u$,\newline
        $f \sim \mathcal{GP}(0, k)$, $v_u \sim \mathcal{N}(0, \sigma^2)$
      & Gaussian
      & $\mu_i = g(t+i) + \kappa_i^\top K^{-1}\big(y_{0:t} - g_{0:t}\big)$,\newline
        $s_i^2 = k(t{+}i, t{+}i) - \kappa_i^\top K^{-1} \kappa_i + \sigma^2$,\newline
        with $K = k(0{:}t, 0{:}t) + \sigma^2 I$ and $\kappa_i = k(t{+}i,\, 0{:}t)$\\
    \addlinespace
    Linear SDE (Ornstein--Uhlenbeck)
      & $dY_u = \kappa\,(\eta - Y_u)\,du + \sigma\,dW_u$
      & Gaussian
      & $\mu_i = \eta + (y_t - \eta)\,e^{-\kappa \delta_i}$, \quad
        $s_i^2 = \dfrac{\sigma^2}{2\kappa}\big(1 - e^{-2\kappa \delta_i}\big)$\\
    \addlinespace
    Log-linear SDE (geometric Brownian motion)
      & $dY_u = \nu\,Y_u\,du + \sigma\,Y_u\,dW_u$
      & Lognormal
      & $\log Y_{t+i} \sim \mathcal{N}(a_i, b_i^2)$ with $a_i = \log y_t + (\nu - \tfrac12 \sigma^2)\delta_i$
        and $b_i^2 = \sigma^2 \delta_i$;\newline
        $\mu_i = e^{a_i + b_i^2/2}$, \quad $s_i^2 = \mu_i^2\big(e^{b_i^2} - 1\big)$\\
    \addlinespace
    i.n.i.d.\ $y_u = g(u, \epsilon_u)$ (ForecastPFN, Sawtooth, Step, Spikes)
      & $g$ known and deterministic, $(\epsilon_u)_u$ independent
      & noise law
      & $Y_{t+i}$ is independent of $y_{0:t}$, so $F_i$ is the marginal law of $g(t{+}i, \epsilon_{t+i})$ and
        $U_i(x) = \int_x^{\infty}\big(1 - F_i(u)\big)\,du$.\newline
        Gaussian when $g(u, \epsilon) = m_u + \sigma_u \epsilon$ with $\epsilon \sim \mathcal{N}(0,1)$,
        giving $\mu_i = m_{t+i}$ and $s_i = \sigma_{t+i}$\\
    \bottomrule
  \end{tabular}

  \vspace{8pt}
  \textbf{(b) Distilled loss $\ell_{distill}$ at horizon index $i$, by conditional law}\\[2pt]
  \begin{tabular}{@{}p{0.115\linewidth}p{0.275\linewidth}p{0.315\linewidth}p{0.235\linewidth}@{}}
    \toprule
    Loss $\ell$
      & Gaussian $\mathcal{N}(\mu, s^2)$\newline $z = (\hat y - \mu)/s$
      & Lognormal $(a, b^2)$\newline $d = (\log \hat y - a)/b$
      & General, via \eqref{eq:musfu}\\
    \midrule
    Squared error
      & $(\hat y - \mu)^2 + s^2$
      & $(\hat y - \mu)^2 + \mu^2\big(e^{b^2} - 1\big)$
      & $(\hat y - \mu)^2 + s^2$\\
    \addlinespace
    Absolute error
      & $s\big[\,2\phi(z) + z\big(2\Phi(z) - 1\big)\big]$
      & $(\hat y - \mu) + 2\big[\mu\,\Phi(b - d) - \hat y\,\Phi(-d)\big]$
      & $(\hat y - \mu) + 2\,U(\hat y)$\\
    \addlinespace
    Pinball at level $\tau$
      & $s\big[\,\phi(z) + z\big(\Phi(z) - \tau\big)\big]$
      & $\mu\,\Phi(b - d) - \hat y\,\Phi(-d) + (1-\tau)(\hat y - \mu)$
      & $U(\hat y) + (1-\tau)(\hat y - \mu)$\\
    \addlinespace
    Cross-entropy over bins
      & $-\sum_k \big[\Phi(z_k) - \Phi(z_{k-1})\big]\log \hat p_k$,\newline $z_k = (e_k - \mu)/s$
      & $-\sum_k \big[\Phi(d_k) - \Phi(d_{k-1})\big]\log \hat p_k$,\newline $d_k = (\log e_k - a)/b$
      & $-\sum_k \big[F(e_k) - F(e_{k-1})\big]\log \hat p_k$\\
    \bottomrule
  \end{tabular}
\end{table}

The lognormal column of Table~\ref{tab:tractable}(b) follows from
$\mathbb{P}(Y > x) = \Phi(-d)$ and $\mathbb{E}[Y \mathbf{1}\{Y > x\}] = \mu\,\Phi(b - d)$, the standard
partial moment of a lognormal, so that $U(x) = \mu\,\Phi(b-d) - x\,\Phi(-d)$. Substituting this into
Proposition~\ref{prop:musfu} gives the remaining three entries. 
The same two ingredients - a CDF and
an upper partial expectation - are all that a new generator family has to supply in order to join
the table.

\subsection{Which Published Synthetic Data Generators are Tractable}
\label{app:generators}

Appendix~\ref{app:closedforms} assumes the conditional law is available. Table~\ref{tab:generators}
records which of the synthetic generators used in recent TSFM pre-training supply it, and which row of
Table~\ref{tab:tractable}(a) each one instantiates given its drawn hyperparameters $\alpha$; for those
rows the distilled losses can be read straight off Table~\ref{tab:tractable}(b).

Two rows need a word. The regime-switching Ornstein--Uhlenbeck generator is closed-form only once the
regime path is recorded at generation time, as Section~\ref{sec:method} describes for the conditional
moments; conditioning on that path in addition to $y_{0:t}$ leaves Proposition~\ref{prop:rb} intact, by
the same tower-property argument. And for the last three rows, where no closed form is available, the
fallbacks of Section~\ref{sec:method} apply: a Monte Carlo estimate of $\ell_{distill}$, or the status
quo loss.

\begin{table}[ht]
  \centering
  \footnotesize
  \setlength{\tabcolsep}{6pt}
  \caption{Synthetic generators used to pre-train TSFMs, and where each lands in
  Table~\ref{tab:tractable}(a). The first three rows are closed-form; the last three are not. Rows one
  to five are the TempoPFN generators \citep{moroshan2025tempopfn}, which subsume KernelSynth
  \citep{ansari2024chronos} and ForecastPFN \citep{dooley2023forecastpfn}; the last row is from Chronos
  \citep{ansari2024chronos}.}
  \label{tab:generators}
  \begin{tabular}{@{}ll@{}}
    \toprule
    Generators & Family in Table~\ref{tab:tractable}(a)\\
    \midrule
    KernelSynth, Gaussian Process & Gaussian process\\
    regime-switching Ornstein--Uhlenbeck & linear SDE, given the recorded regime path\\
    ForecastPFN, Sawtooth, SineWave, Spikes, StepFunction, Anomaly & i.n.i.d.\ $y_u = g(u, \epsilon_u)$\\
    \addlinespace
    CauKer, at non-root channels & none: nonlinear map of GPs, sampling required\\
    audio-inspired (four generators) & none: procedural simulation, sampling required\\
    TSMixup & none: real-data augmentation, law unavailable\\
    \bottomrule
  \end{tabular}
\end{table}

\section{Experimental Details}
\label{app:details}

\paragraph{Generator.} Every series is a single Gaussian process draw on the integer grid
$u = 0, 1, \dots, T-1$. Series are produced in generation chunks of $128$: one kernel $\kappa$ is
drawn uniformly from ten choices per chunk and shared by the whole chunk, while the kernel
hyperparameters $\alpha$, the mean function $g(u) = a u + c$ and the noise realization are drawn per
series. The
per-series marginal over kernels is therefore uniform, but $128$ consecutive series share a kernel,
and the corpus is read in generation order rather than shuffled, so a batch of $B \le 128$
carries a single kernel type and only the $B = 256$ and $B = 1024$ levels of the batch sweep mix
kernels within a batch. Elsewhere kernels vary across steps, not within a step.
Writing
$K_{\kappa\alpha} = \big(k_{\kappa\alpha}(u,v)\big)_{u,v < T}$ for the kernel matrix and
$g = \big(g(u)\big)_{u<T}$, the generating distribution is exactly
\begin{equation}
\label{eq:gp-gen}
y_{0:T-1} \mid (\kappa, \alpha, a, c) \;\sim\; \mathcal{N}\big(g,\; K_{\kappa\alpha} + \tilde\sigma^2 I_T\big),
\qquad \tilde\sigma^2 = \sigma^2 + 10^{-4},
\end{equation}
realized as $y = g + L\xi$ with $LL^\top = K_{\kappa\alpha} + \tilde\sigma^2 I_T$ and
$\xi \sim \mathcal{N}(0, I_T)$. Here $\sigma$ is the observation-noise standard deviation and
$10^{-4}$ is a Cholesky jitter; both are folded into the covariance rather than added afterwards, so
\eqref{eq:gp-gen} is the law the distilled losses condition on. Channels share the kernel and, in the
runs reported here, are drawn independently.

The ten kernels are RBF, Mat\'ern-1/2 (equivalently Ornstein--Uhlenbeck), Mat\'ern-3/2, Mat\'ern-5/2,
periodic (exponentiated sine squared), rational quadratic, locally periodic (RBF times periodic),
linear (dot product), degree-two polynomial, and a three-component spectral mixture. Their
hyperparameters are drawn per series, all uniform: lengthscale $U(5,50)$; output scale $U(0.5,2)$;
period $U(8,64)$; rational-quadratic shape $U(0.5,4)$; periodic lengthscale $U(0.5,2)$; polynomial
offset $U(0,2)$; spectral-mixture weights $U(0.1,1)$, frequencies $U(0.005,0.2)$ and scales
$U(0.001,0.02)$. The mean is linear with probability $0.5$, with slope $a \sim U(-0.02,0.02)$, and
constant ($a = 0$) otherwise; the intercept is $c \sim U(-1,1)$ in both cases.

\paragraph{Distilled losses used.} Equation \eqref{eq:gp-gen} is row two of
Table~\ref{tab:tractable}(a), so the conditional law of the future given the history is Gaussian and
available in closed form. Split the index set at $t$ into the context $c = 0{:}t$ and the horizon
$q = t{+}1{:}t{+}h$, and abbreviate $\widetilde K_{cc} = K_{cc} + \tilde\sigma^2 I$. Then
\begin{equation}
\label{eq:gp-post}
Y_q \mid y_{0:t} \sim \mathcal{N}(\mu, \Sigma), \quad
\mu = g_q + K_{qc}\widetilde K_{cc}^{-1}\big(y_{0:t} - g_c\big), \quad
\Sigma = K_{qq} + \tilde\sigma^2 I - K_{qc}\widetilde K_{cc}^{-1}K_{cq},
\end{equation}
and we write $\mu_i$ and $s_i^2 = \Sigma_{ii}$ for the marginals of \eqref{eq:gp-post}, $i = 1, \dots, h$. Only these
marginals are needed. Substituting them into the Gaussian column of Table~\ref{tab:tractable}(b)
gives the distilled objective for each of the two head types. Every result we report uses
the quantile head; the point head is shown because the method applies unchanged to both:
\begin{align}
\text{point head:}\quad & \ell_{distill} = \frac{1}{h}\sum_{i=1}^{h} \Big[\big(\hat y_{t+i} - \mu_i\big)^2 + s_i^2\Big], \label{eq:gp-mse}\\
\text{quantile head:}\quad & \ell_{distill} = \frac{1}{hK}\sum_{i=1}^{h}\sum_{k=1}^{K} s_i\Big[\phi(z_{ik}) + z_{ik}\big(\Phi(z_{ik}) - \tau_k\big)\Big],
\quad z_{ik} = \frac{\hat y_{t+i,\tau_k} - \mu_i}{s_i}, \label{eq:gp-pin}
\end{align}
with $\tau_1, \dots, \tau_K$ the $K = 9$ deciles. The Status Quo arms replace \eqref{eq:gp-mse} and
\eqref{eq:gp-pin} by the realized squared error and the realized pinball loss against
$y_{t+1:t+h}$. The $s_i^2$ term in \eqref{eq:gp-mse} does not depend on $\theta$ and so does not enter
the gradient; we keep it because it removes the leading level difference between the two arms'
logged losses, making the distilled loss an unbiased estimate of the same quantity as the Status Quo
loss. A smaller weighting artifact survives it, which we describe next.

Under teacher forcing the same formulae apply at $h = 1$ with $t$ running along the series. There the
conditioning is available for free from the draw: with $L$ the lower-triangular Cholesky factor of
$K_{\kappa\alpha} + \tilde\sigma^2 I$, $\mu = y_t - L_{tt}\,\xi_t$ and $s = L_{tt}$, so no per-step
solve is needed.

Within a head, the two arms are scored by identical code. The quantile arms are scored in the
model's normalized space; the point arms are de-normalized and then re-normalized by a second,
detached per-series mean and standard deviation. Held-out metrics are therefore comparable across
arms. The logged training objectives of the CPM pair are not exactly comparable even after the
$s_i^2$ correction above, because the per-series weight they divide by is a standard deviation
measured over the full series, including the masked span: it is inflated by the span's own realized
excursion and so correlates with the Status Quo arm's squared residual while leaving the distilled
arm's constant variance term untouched. The residual level difference is small, but it is not zero,
which is why every cross-arm number we report is a held-out metric rather than a training objective.

\paragraph{Model.} Every result in this paper uses \textbf{Toto-2} \citep{toto2}, a next-patch
predictor at patch length $P = 32$: the patch at index $k$ predicts patch $k+1$, so the objectives
\eqref{eq:gp-mse}--\eqref{eq:gp-pin} and the held-out metrics below apply without modification.

The five architectures are rebuilt verbatim from the published configurations, at
$4.1$M, $21.9$M, $312.7$M, $1041.0$M and $2454.3$M parameters, and trained from random
initialization rather than from released weights, so that both arms of a comparison start from the
same untrained model. Each interleaves causal time attention over patches with variate attention
across channels, and is unit-scaled, its parameters initialized at unit magnitude with the scaling
carried in the forward pass. 
The point-estimate arm reads the median
knot of the model's own output; the quantile arm reads the nine deciles directly.

When a contiguous patch mask is applied, the masked entries are hidden from the model's scaler, so
the values to be predicted cannot leak into the normalization statistics, and the resulting location and scale bypass the trunk. The quantile arm's deciles are emitted, and scored, in that normalized space. This is what makes the objective comparison well posed: the two arms of a run differ only in the target their loss is taken against, and are evaluated by identical held-out metric code.

\paragraph{Training.} AdamW, weight decay $10^{-4}$, gradient-norm clipping at $1.0$, in bf16
mixed precision, in every run, at a learning rate of
$10^{-5}$, warmed up linearly over the first $1{,}000$ steps and then decayed by a cosine schedule to the end of training. 
Full Toto-2 pre-training \cite{toto2} used a unit-scaled optimizer that multiplies the learning rate per parameter by $1/\sqrt{\mathrm{fan\text{-}in}}$, a factor of $16$ at $4$M rising to $45$ at $2.5$B, so its published rates are $O(10^{-2})$ and are not comparable with a plain AdamW rate. We train under plain AdamW, which applies the same relative update per step at every width, at a rate
that is stable across all five sizes. Both arms of every comparison share the optimizer exactly, so
the choice sets the level of the curves in Figure~\ref{fig:flops-crps} but not the gap between them.
Contiguous patch masking (CPM) masks one contiguous span of patches and
scores the loss on the masked entries; the span start and length are resampled at every step. The
length is uniform on $\{1, \dots, \min(16, \lfloor 0.4 N \rfloor, N-1)\}$ and the start is uniform
over the positions that leave a prefix of at least one unmasked patch, so at $T = 512$ and $P = 32$,
i.e.\ $N = 16$ patches, the span is one to six patches long and never begins at patch $0$. One span
is drawn per optimizer step and shared by every series and every channel in the batch. Teacher
forcing instead scores every next-patch prediction along the series.

\paragraph{Regimes.} Every run uses one channel ($C = 1$) and series length $T = 512$ unless a sweep
varies it, and there are two studies.

The \emph{Toto-2 streaming study} backs Section~\ref{sec:experiments} and Appendix~\ref{app:tf}.
It runs $300{,}000$ optimizer steps at $B = 16$ and $\sigma = 0.25$, drawn in order from a corpus of
$6{,}000{,}000$ series, so the $4{,}800{,}000$ trajectories it consumes are each used once and none
is revisited. All five Toto-2 sizes are trained under both objectives at ten paired seeds, with an
evaluation every $50$ steps; the contiguous-patch-masking arms are those of
Section~\ref{sec:experiments} and the teacher-forcing arms those of Appendix~\ref{app:tf}.
Cumulative training compute is estimated as $6ND$ with $N$ the measured
parameter count and $D$ the patch tokens consumed, $B \cdot C \cdot T / P = 256$ per step; the
attention term is dropped, being well under a percent at a $16$-patch sequence.

The \emph{factor sweeps} of Appendix~\ref{app:factors} are a separate, shorter study on the same
generator and the Toto-2 $313$m architecture, so their levels are comparable with one another but
not with Section~\ref{sec:experiments}. The noise and sequence-length sweeps run $100{,}000$ steps
at $B = 256$, over $\sigma \in \{0.1, 0.25, 0.5, 1, 5, 10\}$ and
$T \in \{128, 256, 512, 1024\}$ respectively, everything else held at the centre
($\sigma = 0.25$, $T = 512$), which the two share as a single run.

The batch sweep, $B \in \{64, 128, 256, 512\}$, runs $50{,}000$ steps at every level, so all four
share one schedule and only $B$, and the $50{,}000 \cdot B$ series it implies, vary. Fixing the
\emph{data} instead --- one pass over a common corpus, so a level runs $1/B$ as many steps --- would
confound batch size with update count, which Appendix~\ref{app:factors} quantifies. Every level of
every sweep draws from a $26{,}000{,}000$-series corpus and consumes at most $25.6$M of it, so no
trajectory is revisited. The learning rate is held at $10^{-5}$ and not scaled with $B$; both arms of
a level share it, so each reported speed-up is a like-for-like comparison.

The two masking schemes share a model family, a metric, a budget and a seed set, so
Section~\ref{sec:experiments} and Appendix~\ref{app:tf} may be read against each other. The factor
sweeps may not: they run to a third of the budget, and the batch sweep to a sixth, so their numbers
should be read within that study.

\paragraph{Metrics.} \emph{Held-out} throughout means a validation set drawn fresh from the same
generator, using generator seeds disjoint from those that produced the training data, rather than a
held-out split of the training corpus.
Next-token CRPS is the headline metric: $2/K$ times the sum of the pinball losses over the $K = 9$
deciles, in the model's normalized space, evaluated on the one-patch-ahead prediction of a single
\emph{unmasked} teacher-forcing pass, in which the model sees the whole series and every output
position is scored against the patch that follows it. It is the only metric reported in this paper,
and it is scored identically for the contiguous-patch-masking and teacher-forcing arms, neither of
which is evaluated under its own training-time masking.
Section~\ref{sec:experiments} calls the next-token CRPS the validation loss; we write \emph{gap} for
the paired difference between arms, SDD minus Status Quo, at a given step, so a negative gap favors
SDD. Every metric here is scored by identical code for both arms, so a cross-arm difference is a
quality comparison.

\paragraph{Seeding.} Ten seeds per configuration in the streaming study, three in the factor
sweeps. A seed enters an average only if both of its arms
completed the full step budget, a partially trained arm's final value being a mid-training loss.
Within a seed the two arms share an initialization
and a data order, so the comparison is paired; the gap is computed per seed and then averaged.
The seed sets only the global PyTorch generator, so what varies across seeds is the initialization
and the CPM span draws; the training corpus is a fixed cache and is read in the same order by every
seed. The error bars therefore carry no data-sampling variability and understate the spread that
retraining on a freshly drawn corpus would show.

\section{Additional Results}
\label{app:extra}

The two studies below reuse the setups of Appendix~\ref{app:details}: the Toto-2 streaming study
under teacher forcing rather than contiguous patch masking, and the factor sweeps.
Appendix~\ref{app:tf} reports the \emph{gap} as defined in Appendix~\ref{app:details}, so a negative
number favors SDD; Appendix~\ref{app:factors} reports speed-ups instead.

\subsection{Teacher Forcing, Streaming}
\label{app:tf}

The setting is exactly that of Section~\ref{sec:experiments} -- the same five Toto-2 architectures,
the same quantile head over the nine deciles, the same $300{,}000$ streaming steps at $B = 16$ and
$\sigma = 0.25$, and the same ten paired seeds -- with the training objective scored under teacher
forcing at every next-patch position rather than on a contiguous masked span. The metric is the same
held-out next-token CRPS, so the numbers below are directly comparable with those of
Section~\ref{sec:experiments}.

Figure~\ref{fig:tf} does not repeat the layout of Figure~\ref{fig:flops-crps}, because at this effect
size that layout cannot show the result. Drawn on one shared axis, the five sizes span a range of
CRPS a hundred times the distance between the two arms, and the pairs are separated by less than the
width of the lines drawing them: the two frontiers coincide and the figure reads as a null result.
Splitting the sizes into their own panels is what makes the comparison visible, since each panel can
then be cropped to one model's own converged band.

SDD still wins every one of the fifty paired seeds, but the effect is an order of magnitude smaller:
the converged gap runs from $-0.11\%$ to $-0.41\%$, against $-0.90\%$ to $-2.16\%$ under contiguous
patch masking, and the speed-up is $1.19\times$ to $1.33\times$ rather than $1.6\times$ to
$1.85\times$. The ordering is what the variance-reduction account of Section~\ref{sec:theory}
predicts. Teacher forcing scores every position of the series at each step, so a single gradient
already averages the sampling noise of the realized future over all $T$ time points; contiguous patch
masking scores only the masked span, at one to six patches of sixteen, so its per-step gradient is
built from far fewer target draws and carries correspondingly more of the noise SDD removes. Where
gradients are already well averaged, there is less variance left to take out -- the same mechanism
the batch-size panel of Appendix~\ref{app:factors} exhibits along a different axis.

Teacher forcing also reaches a lower absolute CRPS than masking at every size, which is expected:
the unmasked next-token metric matches its training objective, while the masked-span arms are scored out of the regime they train in.

\begin{figure}[ht]
  \centering
  \includegraphics[width=\linewidth]{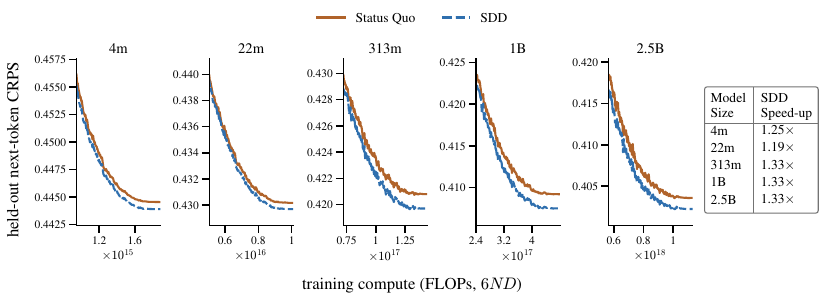}
  \caption{Teacher forcing rather than contiguous patch masking, one panel per architecture:
  held-out next-token CRPS against cumulative training compute, means over ten paired seeds. Colour
  carries the training objective here, not the model size, which is the panel title. Each panel is
  drawn on its own compute and CRPS scale and covers the second half of training, the range over
  which both arms have converged; \emph{the panels are therefore not comparable with one another},
  and neither their vertical scales nor their compute axes are shared. The table at the right gives
  each size's speed-up, defined as in Figure~\ref{fig:flops-crps}. The gap SDD opens early is partly closed by
  Status Quo by the end of training, which is why the curves converge towards the right of every
  panel and why the converged gaps below are smaller than the separation visible mid-panel.}
  \label{fig:tf}
\end{figure}

\subsection{Factor Sweeps}
\label{app:factors}

Section~\ref{sec:experiments} varies the model size and holds everything else fixed.
Table~\ref{fig:factors} varies each of the other three factors in turn, on the Toto-2 $313$m
quantile head, and reports each level as an SDD \emph{speed-up}: the compute Status Quo spends to
reach its final held-out next-token CRPS, over the compute SDD needs to reach the same value. A
speed-up of $1.3\times$ means SDD arrives at Status Quo's converged accuracy on three quarters of
the budget. We report the speed-up rather than the converged gap because it is what a practitioner
spends: near convergence the loss curve is flat, so a gap of a few tenths of a percent is still
worth tens of percent of the compute.

The noise and sequence-length columns share an operating point --- $B = 256$ for $100{,}000$ steps
at $\sigma = 0.25$ --- and meet at a common level, $\sigma = 0.25$ and $T = 512$, which is one run
and appears in both at $1.32\times$. The batch column holds the step count at $50{,}000$ instead, so
that only $B$ and the data it implies vary, and is read down its own levels. Three paired seeds per
level.

\emph{Observation noise} costs SDD little: $1.32\times$ at $\sigma \le 0.25$, easing to $1.22$--$1.23
\times$ from $\sigma = 1$ and then flat out to $\sigma = 10$. The decline is real but small over two
decades. Observation noise is irreducible for both arms, so raising it inflates the CRPS each can
attain faster than the difference between them; but it also flattens both curves, so the shrinking
accuracy gap still buys a similar share of the compute.

\emph{Sequence length} works the other way, rising monotonically with context: $1.26\times$ at
$T = 128$ to $1.34\times$ at $1024$. Longer contexts make the conditional distribution sharper and
its moments more informative, so the distilled target carries more signal per step.

\emph{Batch size} falls monotonically, $1.26\times$ at $B = 64$ to $1.13\times$ at $512$, and is the
factor the theory speaks to most directly. SDD removes the per-sample sampling noise of the realized
future, but averaging over a batch already suppresses that noise by roughly $B$, so the larger the
batch the less there is left to remove. The variance reduction is worth most where gradients are
noisiest.

Holding the step count fixed here rather than the data is deliberate. Under equal data --- one pass
over a common corpus, so a level runs $1/B$ as many steps --- the same effect looks like a collapse
rather than a decay, from $-3.5\%$ at $B = 64$ to $-0.01\%$ at $B = 1024$. Most of that is update
count, not batch size: $B = 256$ takes a quarter of the steps $B = 64$ does. Fixing the steps
isolates the batch effect, and what survives is the milder decline above.

\begin{table}[ht]
  \centering
  \includegraphics[width=\linewidth]{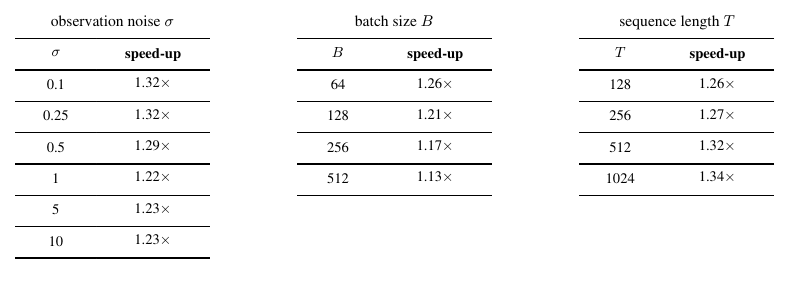}
  \caption{Toto-2 $313$m quantile head: the SDD speed-up at each level of three swept factors,
  defined as in Figure~\ref{fig:flops-crps} --- the compute Status Quo spends to reach its converged
  held-out next-token CRPS, over the compute SDD needs to reach the same value. Higher favors SDD.
  The noise and sequence-length columns run $B = 256$ for $100{,}000$ steps at $\sigma = 0.25$ and
  meet at $\sigma = 0.25$, $T = 512$; the batch column runs $50{,}000$ steps per level so that only
  the batch varies. Each column is read down its own levels. Every level consumes $25.6$M of a
  $26$M-series corpus, so no trajectory is revisited.}
  \label{fig:factors}
\end{table}

\end{document}